\documentclass{article}

\usepackage{graphicx}
\usepackage{subfigure}
\usepackage{amsmath,amssymb}

\usepackage{latexsym}
\usepackage{verbatim}
\usepackage[usenames,dvipsnames]{color}
\usepackage{wrapfig}
\usepackage{float,times}
\usepackage{amssymb}
\usepackage{stackengine}

\begin{document}

\title{The Prediction Advantage: A Universally Meaningful Performance Measure for Classification and Regression}

%

\author{
	Ran El-Yaniv\\
	Technion – Israel Institute of Technology\\
	\texttt{rani@cs.technion.ac.il} \\
		Yonatan Geifman\\
	Technion – Israel Institute of Technology\\
	\texttt{yonatang@cs.technion.ac.il} \\
	Yair Wiener\\
	Jether Energy Research\\
	\texttt{yair@jether-energy.com} \\
}
\maketitle





%



\newtheorem{theorem}{Theorem}[section]
\newtheorem{lemma}[theorem]{Lemma}

\newenvironment{proof}[1][Proof:]{\begin{trivlist} \item[\hskip \labelsep
{\bfseries #1}]}{\hfill{$\Box$}\end{trivlist}}

\begin{abstract}
	We introduce the \emph{Prediction Advantage} (PA), a novel performance measure 
	for prediction functions under any loss function (e.g., classification or regression).
	The PA is defined as the performance advantage relative to 
	the Bayesian risk restricted to knowing only the distribution of the labels.
	We derive the PA for well-known loss functions, including  0/1 loss, cross-entropy loss,
	absolute loss, and squared loss. In the latter case, the PA is identical to the 
	well-known $R$-squared measure, widely used in statistics.
	The use of the PA ensures meaningful quantification of prediction performance, which is not guaranteed, for example, when dealing with noisy imbalanced classification  problems.
	We argue that among 
	several known alternative performance measures, PA is the best (and only) quantity  
	ensuring meaningfulness for all noise and imbalance levels.
\end{abstract}

%
%
%
%


\newcommand{\rnote}[1]{ {\color{red}  [ Ran: \textit{#1}]} }
\newcommand{\ynote}[1]{ {\color{blue} [ Yonatan: \textit{#1}]} }
\newcommand{\sharon}[1]{ {\color{blue} {#1}} }

\newcommand{\argmax}{\mathop{\mathrm{argmax}}\limits}
\newcommand{\argmin}{\mathop{\mathrm{argmin}}\limits}
\newcommand{\maxlim}{\mathop{\mathrm{max}}\limits}
\newcommand{\vect}[1]{\mathbf{#1}}
\newcommand{\one}{\mathbf{1}}

\newcommand{\VS}{\mathcal{VS}} 

\newcommand{\data}[1]{\texttt{#1}}
\newcommand{\tdata}[1]{\scriptsize{\texttt{#1}}}
\newcommand{\COIL}{\data{COIL}}
\newcommand{\MUSH}{\data{MUSH}}
\newcommand{\MUSK}{\data{MUSK}}
\newcommand{\PIMA}{\data{PIMA}}
\newcommand{\BUPA}{\data{BUPA}}
\newcommand{\VOTING}{\data{VOTING}}
\newcommand{\MONK}{\data{MONK}}
\newcommand{\IONO}{\data{IONOSPHERE}}
\newcommand{\TAE}{\data{TAE}}
\newcommand{\DIGIT}{\data{DIGIT}}
\newcommand{\TEXT}{\data{TEXT}}

\newcommand{\sCOIL}{\tdata{COIL}}
\newcommand{\sMUSH}{\tdata{MUSH}}
\newcommand{\sMUSK}{\tdata{MUSK}}
\newcommand{\sPIMA}{\tdata{PIMA}}
\newcommand{\sBUPA}{\tdata{BUPA}}
\newcommand{\sVOTING}{\tdata{VOTING}}
\newcommand{\sMONK}{\tdata{MONK}}
\newcommand{\sIONO}{\tdata{IONOSPHERE}}
\newcommand{\sTAE}{\tdata{TAE}}
\newcommand{\sDIGIT}{\tdata{DIGIT}}
\newcommand{\sTEXT}{\tdata{TEXT}}

\newcommand{\myalg}[1]{\texttt{#1}}
\newcommand{\mytalg}[1]{\scriptsize \texttt{#1}}
\newcommand{\SGT}{\myalg{SGT}}
\newcommand{\ZHU}{\myalg{GRMF}}
\newcommand{\BELKIN}{\myalg{SOFT}}
\newcommand{\SCHOLKOPF}{\myalg{CM}}
\newcommand{\EXPP}{\myalg{+EXPLORE}}
\newcommand{\sEXPP}{\mytalg{+EXPLORE}}
\newcommand{\sZHU}{\mytalg{STRICT}}
\newcommand{\sBELKIN}{\mytalg{SOFT}}
\newcommand{\sSCHOLKOPF}{\mytalg{RANGE}}
\newcommand{\sSGT}{\mytalg{SGT}}

\newcommand{\ENG}{\myalg{ENERGY}} 

\newcommand{\KNN}{\myalg{kNN}}

\newcommand{\comp}[1]{\small \texttt{#1}}
\newcommand{\QEXPLORE}{\comp{Q-EXPLORE}}
\newcommand{\QEXPLOIT}{\comp{Q-EXPLOIT}}
\newcommand{\EXPPSWITCH}{\comp{EXPLORE-EXPLOIT-SWITCH}}

\newcommand{\cA}{{\cal A}}
\newcommand{\cB}{{\cal B}}
\newcommand{\cF}{{\cal F}}
\newcommand{\cG}{{\cal G}}
\newcommand{\cH}{{\cal H}}
\newcommand{\cL}{{\cal L}}
\newcommand{\cV}{{\cal V}}
\newcommand{\cX}{{\cal X}}
\newcommand{\hL}{{\widehat{L}}}
\newcommand{\hcL}{{\widehat{\mathcal{L}}}}
\newcommand{\hR}{\widehat{R}}
\newcommand{\bB}{\mathbf{B}}
\newcommand{\bW}{\mathbf{W}}
\newcommand{\bR}{\mathbf{R}}
\newcommand{\bD}{\mathbf{D}}
\newcommand{\bG}{\mathbf{G}}
\newcommand{\bL}{\mathbf{L}}
\newcommand{\bI}{\mathbf{I}}
\newcommand{\bC}{\mathbf{C}}
\newcommand{\bZ}{\mathbf{Z}}
\newcommand{\ba}{\mathbf{a}}
\newcommand{\bd}{\mathbf{d}}
\newcommand{\be}{\mathbf{e}}
\newcommand{\bg}{\mathbf{g}}
\newcommand{\Bf}{\mathbf{f}}
\newcommand{\Bg}{\mathbf{g}}
\newcommand{\bh}{\mathbf{h}}
\newcommand{\bp}{\mathbf{p}}
\newcommand{\bq}{\mathbf{q}}
\newcommand{\bt}{\mathbf{t}}
\newcommand{\bu}{\mathbf{u}}
\newcommand{\bv}{\mathbf{v}}
\newcommand{\bx}{\mathbf{x}}
\newcommand{\by}{\mathbf{y}}
\newcommand{\bz}{\mathbf{z}}
\newcommand{\E}{\mathbf{E}}
\newcommand{\var}{\text{var}}
\renewcommand{\H}{\mathbf{H}}
\renewcommand{\S}{\mathbf{S}}
\newcommand{\T}{\mathbf{T}}
\newcommand{\CM}{\mathtt{CM}}
\newcommand{\CMRAD}{\mathtt{CM-SUP}}
\newcommand{\eqdef}{\mathrel{\ensurestackMath{\stackon[1pt]{=}{\scriptstyle\Delta}}}}

\newcommand{\nchoosek}[2]{\left(\begin{array}{c}#1\\#2\end{array}\right)}
\renewcommand{\Pr}{\mathbf{Pr}}
\newcommand{\head}{\mathrm{head}}
\newcommand{\tail}{\mathrm{tail}}
\newcommand{\rad}{\mathtt{R}}
\newcommand{\parr}{\mathtt{Par}}
\newcommand{\tA}{\mathtt{A}}
\newcommand{\bpi}{\pmb{\pi}}
\newcommand{\cN}{{\mathcal N}}
\newcommand{\cY}{{\mathcal Y}}
\newcommand{\rE}{\mathbf{E}}
\newcommand{\pig}{\tilde\pi_{\gamma/k}}
\newcommand{\al}{\alpha}
\newcommand{\QED}{\hfill{$\Box$}}
\newcommand{\lam}{\lambda}
\newcommand{\err}{\mathop{\rm er}}

\newcommand{\I}{\mathbb{I}}
\newcommand{\hf}{f_Q}
\newcommand{\hg}{\hat{g}}
\newcommand{\LA}{{\mathcal L }}
\newcommand{\GLi}{G_L^{(i)}}
\newcommand{\GUi}{G_U^{(i)}}
\newcommand{\fail}{\texttt{fail}}
\newcommand{\reals}{\mathbb{R}}

\section{Introduction}
\label{sec:introduction}

Consider the task of training a classifier on a binary problem with a very small minority class 
whose proportion is 1\%. A machine learning intern generates a classifier whose test accuracy was 97\%. Proud of this excellent result, the intern reports to his boss and is promptly fired. 
Many of us have encountered this embarrassing situation where, in retrospect, the classifier 
we have worked hard to train, turns out to be no better than the trivial classifier that 
labels everything as the majority class (achieving 99\% accuracy in our example).

While this example was intentionally contrived, it occurs widely, even with less extreme imbalance.
For example, the proportion of the minority class in the well-known UCI Haberman dataset 
\cite{Lichman:2013}
is 26.47\% and yet it fails even experienced ML researchers 
\cite{wiener2011agnostic,huang2007correcting,xiong2012random,chandra2009fuzzifying,ghanavati2014effective,mcconnell2004building,juhola2013missing,chandra2006robust,alabdulmohsin2014support}.  In all these publications, the authors have reported the (positive) results
of a classifier for the Haberman set whose error is near or worse than the trivial 26.47\% error
(see more details in Appendix~\ref{app:haberman}).
Prediction under other loss functions 
(not the 0/1 loss) is also susceptible to a qualitatively similar problem. 

Known partial remedies for handling imbalanced datasets include 
the $F$-measure, true positive/negative rates, and Cohen's kappa.
These measures can sometimes detect trivial classifiers, but not always. They are also 
not general in the sense that they are not defined for any desired loss function.
For example, Cohen's kappa is defined only for the 0/1 loss (i.e., classification).
Moreover, each of these known measures 
can still fail a machine learning practitioner in some problems.
In the context of squared loss regression, the well-known $R^2$ measure provides an effective 
solution (and, as will be shown below, $R^2$ is a special case of our measure).

In this paper we propose a single performance measure, potentially applicable to any loss function.
Consider a learning problem defined in terms of a feature space $\cX$, a label space $\cY$,  a  distribution $P(X,Y)$, where $X \in \cX$, $Y \in \cY$, and 
a loss function $\ell : \cY \times \cY \to \reals^+$.
To measure the performance of a prediction function $f$, the proposed measure considers the advantage of the risk $R_{\ell}(f) \eqdef \E_{X,Y} \{\ell(f(X),Y) \}$ over the risk of the best  prediction given knowledge
of the marginal $P(Y)$, 
which we term the \emph{Bayesian marginal prediction} (BMP) risk. 
The new measure is termed the \emph{prediction advantage} (PA).
We derive the PA for the most popular loss functions, namely, 0/1 loss, cross-entropy loss, squared loss, absolute loss, and general cost-sensitive loss.
We then argue that the PA measure always prevents triviality whereas each of the other known alternative measures can fail.

A striking benefit of the
PA is that it also enables performance comparisons across learning problems, e.g.,
in classification problems
that differ in the number of classes (and imbalance rate).
To understand this effect, consider the following story.
The intern from our first vignette becomes a university professor
and composes a multiple choice  exam with 100 questions for his machine learning class. 
To prevent cheating, he writes two versions of the exam, both containing precisely the same questions
but version A has 3 optional answers for each question and version B, 4 optional answers.
Bob and Alice write the exam and receive versions  A and B, respectively. 
Their exam sheets are graded and both get a grade of 30. It is easy to see that Bob's grade is 
worse than the trivial 33.33\% mark that can be achieved for a guess (for version A), while
Alice's is better than the corresponding trivial 25\% mark for version B. Using PA, however, not only does the PA detect Bob's trivial achievement, it also allows us to quantify 
his performance compared Alice's. For example, it enables us to determine how much better Alice's performance is than Bob's
if they both got a mark of 60\% (see analysis in Section~\ref{sec:experiments}).

\section{Prediction Advantage}
\label{sec:predictionadvantage}

In this section we introduce the Bayesian marginal prediction (BMP), defined to be the optimal 
prediction in hindsight given the marginal distribution of the labels, $P(Y)$.
For any given loss function, the risk of the BMP is taken as a reference point for meaningless prediction, which we would like outperform. To this end, we define the prediction advantage (PA) as the (additive) reciprocal of
the performance ratio between our prediction function and the BMP.
We then instantiate the PA to several important loss functions.


\subsection{Bayesian Marginal Prediction}

We define the Bayesian Marginal Prediction function (BMP) to be the optimal prediction function with respect to the marginal distribution of $Y$, $P(Y)$, and denote it as 
$f_0$.\footnote{In the case that the BMP is not unique, we will choose it arbitrarily among the optimal, since we are interested in its risk.} The BMP predicts a constant value/class while being oblivious
to $X$ and $P(Y|X)$. In this way we expect the BMP to obtain only the complexity of the problem latent in $P(Y)$. 
We denote by $\cB$ the set of all probability distributions over $\cY$. Clearly, 
any BMP must reside in $\cB$. 

For any prediction function $g \in \cB$, and any loss function $\ell$,  we have 
$R_\ell(g) = \E_Y \ell(Y, g)$.  This can be easily established by noting that $g$ is independent of $X$:
$$
\E_{X,Y} \ell(Y, g) = \int_{X,Y} \ell(Y, g) d P(X,Y) = \int_Y  \ell(Y, g) \int_X   d P(X,Y) 
= \int_Y  \ell(Y, g) d P(Y) = \E _Y  \ell(Y, g).
$$
We will use this simple relation throughout the paper.

By Yao's principle \cite{Yao:1977:PCT:1382431.1382556,borodin1999randomization}
(which follows from von Neumann's minimax theorem), 
we can restrict attention only to deterministic BMPs (i.e., constants).
For self-containment, we provide a direct proof of this statement for the case where the loss function is convex in its second argument.
\begin{lemma}
	\label{lemma:constant}
	Let $\ell(r,s)$ be a loss function convex with respect to $s$.
	There exists a constant prediction function $f_0 \in \cB$ such that 
	$$
	R_{\ell}(f_0) \leq \min_{g \in \cB} R_{\ell}(g).
	$$
\end{lemma}
\begin{proof}
	Let $g$ be any function in $\cB$ defined with respect to the distribution $Q$ over $\cY$.
	Clearly,
	$R_\ell(g)= \E_{Y}  \E_Q \{  \ell(Y,g(X))  \}$.
	Using the Jensen inequality we obtain the desired bound as follows:
	\begin{eqnarray*}
		R_\ell(g)&=&\E_{Y}\E_{Q}\ell(Y,g)\\
		&\geq&\E_{Y}\ell(Y,\E_{Q}\{g\})\\
		&=&R_\ell(\E_{Q}\{g\}).
	\end{eqnarray*}
	Thus, we have shown that for any $g = g_Q \in \cB$, the risk of the constant 
	prediction $\E_{Q}\{g\}$ is no worse than  $R_\ell(g)$.
	This also holds for the best prediction function in $\cB$.
\end{proof}
For optimality, the BMP will be defined explicitly according to a specific loss function. We will define it and prove its optimality for several loss functions.

\subsection{The Prediction Advantage}
\label{sec:prediction advantage}

After defining the BMP, we will use its risk as the baseline for our measure. The Prediction Advantage (PA) of a prediction function $f$ is defined to be the advantage of the expected performance of $f$ over the BMP:
$$\textnormal{PA}_\ell (f)\eqdef1-\frac{R_\ell(f)}{R_\ell(f_0)}=1-\frac{E_{X,Y}(\ell(f(X),Y))}{E_{X,Y}(\ell(f_0(X),Y)}.$$

The following are basic properties of the PA measure:
\begin{enumerate}
	\item \textbf{Order preservation:} The PA forms a weak ordering of the functions $f \in \cF$ (given any function class 
	$\cF$ for our problem). Since the prior probabilities over the classes of a problem are constant, the PA preserves an order that is inverse to the order formed by the risk.
	$$
	\textnormal{PA}_\ell(f_1)>\textnormal{PA}_\ell(f_2) \textnormal{ iff } R_\ell(f_1)<R_\ell(f_2).
	$$
	\item \textbf{Boundedness:} $\sup_{f\in F}(\textnormal{PA}_\ell(f))=1$; the PA is bounded by 1. Notice that $\textnormal{PA}_\ell(f)=1$ is obtained only when the $R_\ell(f)=0$.
	\item \textbf{Meaningfulness:} $\textnormal{PA}_\ell(f)=0$  when $f$ has no advantage over the BMP. $\textnormal{PA}_\ell(f)<0$ when $f$ is worse than the BMP. A function with a negative PA is meaningless in the sense that one can use the (empirical) BMP, potentially reaching better results.
	
\end{enumerate}

We emphasize that throughout the paper we consider the ``true prediction advantage'', which corresponds to the underlying \emph{unknown} distribution.  The empirical version of 
the PA is straightforwardly defined. To calculate the empirical PA, 
all we need to do is to estimate the BMP (using appropriate estimators) based on a labeled 
sample. Confidence levels for these 
estimates can be obtained using standard concentrations of measure bounds.
For simplicity, we deliberately ignore in this paper the empirical PA, and these estimation problems.

\subsection{Prediction Advantage for the Cross-Entropy Loss}
\label{sec:pa for cross-entropy}
We now look at a multi-class classification problem using the cross-entropy loss defined as, $\ell(f(X),Y)=-\sum_{i\in C}{\Pr\{Y=i\}\log{(\Pr\{f(X)=i\})}}.$
This loss function is extensively used in deep learning and other multi-class learning frameworks.
Having defined the PA in Section~\ref{sec:prediction advantage}, in this section we identify the 
BMP and present an explicit expression for the PA in this setting.
Let $C \eqdef \{ 1, 2, \ldots, k\}$ be the set of the classes and consider a multi-class
classifier, $f(x): \cX \to \reals^k$, whose $i$th coordinate, $f(x)_i$, attempts to predict the probability that an instance $x$ belongs to class $i \in C$; that is, $f(x)_i \eqdef \Pr \{Y = i | x \}$. 
For simplicity, we assume from now on that the label $Y$ is represented in unary; for example,
the label  $Y=i$ is represented by the unit vector $e_i$ (also known as ``one-hot'' encoding). In this setting the cross-entropy loss is defined as $\ell(f(x),y)\eqdef -\sum_{i\in C}y_i\cdot\log(f(x)_i))$.

\begin{lemma}[BMP for cross-entropy loss]
	\label{lemma:crossentropy}
	Let $f_0 = f_0(X)$ be the vector whose $i$th coordinate is $f_0(X)_i\eqdef \Pr\{Y=e_i\}$. 
	Then $f_0$ is the BMP for the multi-class classification under the cross-entropy loss. 
	Moreover, the BMP risk is the Shannon entropy of $P(Y)$, the marginal distribution of Y.
\end{lemma}
\begin{proof}
	Consider an arbitrary marginal classifier $f_Q(X)$ defined with respect to the 
	probability mass function $Q$ whose support set is the set of classes 
	$C$; that is,  $f_Q(X) \sim Q$. We prove that $R_\ell(f_0) \leq R_\ell(f_Q)$.  
	Applying the cross-entropy loss on  $f_0$, we have
	\begin{eqnarray*}
		R_{\ell}(f_0)&=&\E\ell(f_0(X),Y)\\
		&=&\sum_{i\in C}\Pr\{Y=e_i\}\ell(f_0(X),e_i)\\
		&=&\sum_{i\in C}{-\Pr\{Y=e_i\}\log{(\Pr\{Y=e_i\})}}\\
		&=&H(Y).
	\end{eqnarray*}
	The risk of  $f_Q$ is
	\begin{eqnarray*}
		R_{\ell}(f_Q)&=&\E\ell(f_Q(X),Y)\\
		&=&\sum_{i\in C}\Pr\{Y=e_i\}\ell(f_Q(X),e_i)\\
		&=&\sum_{i\in C}{-\Pr\{Y=e_i\}\log{(f_{Qi}(X))}}.
	\end{eqnarray*}
	Using the non-negativity of the Kullback-Leibler divergence, we show that the difference 
	$R_{\ell}(f_Q)-R_{\ell}(f_0)$ is not negative (meaning that $f_0$ is optimal):
	\begin{eqnarray*}
		R_{\ell}(f_Q)-R_{\ell}(f_0)&=&\sum_{i\in C}{-\Pr\{Y=e_i\}\log{(f_{Qi})}}+\sum_{i\in C}{\Pr\{Y=e_i\}\log{(\Pr\{Y=e_i\})}}\\
		&=&\sum_{i\in C}{\Pr\{Y=e_i\}\log{(\Pr\{Y=e_i\}/f_{Qi}(X))}}\\
		&=&D_{kl}(f_0(X)||f_Q(X))\\
		&\geq& 0.
	\end{eqnarray*}
\end{proof}

By Lemma~\ref{lemma:crossentropy}, the risk of the BMP  in the cross-entropy loss setting is
$R_{\ell}(f_0)=H(P(Y))$, and therefore,
the Prediction Advantage for the cross-entropy loss is
$$
\textnormal{PA}_\ell (f)\eqdef1-\frac{R_\ell(f)}{H(P(Y))} .
$$

\subsection{Prediction Advantage for 0-1 Loss}
Consider a 0-1 multi-class classification problem in which the set of classes is $C \eqdef \{ 1, 2, \ldots, k\}$.
In this case (and in contrast to the cross-entropy case discussed above), the classifier $f$
predicts  a nominal value, $f(x) : \cX \to C$. 
The 0-1 loss is defined as $\ell(\hat{y},y)=I(y\neq\hat{y})$, where $I$ is the indicator function. 
Clearly, when the 0-1 loss function is used, the resulting risk equals the probability of misclassification.
To derive the PA for this setting, we now show that the BMP strategy is
$f_0\eqdef\argmax_i(\Pr\{Y=i\})$; that is, the BMP outputs the most probable 
class in $P(Y)$ for every $X$. 

\begin{lemma}[BMP for 0-1 loss]
	\label{0-1lemma}
	For (multi-class) 0-1 loss classification, the Bayesian marginal prediction is 
	$f_0\eqdef\argmax_i((\Pr\{Y=i\}))$.
\end{lemma}
\begin{proof}
	Consider an arbitrary marginal classifier $f_Q(X)$ defined with respect to the 
	probability mass function $Q$ whose support set is the set of classes 
	$C$; that is,  $f_Q(X) \sim Q$. We prove that $R_\ell(f_0) \leq R_\ell(f_Q)$. We denote the most probable class in P(Y) as $j$. 
	Applying the 0-1 loss on  $f_0$, we have
	
	$$
	R_{\ell_{0-1}}(f_0)=1-\max_{i\in C}(\Pr\{Y=i\})=1-\Pr\{Y=j\}.
	$$
	The risk of $f_Q$ is
	$$
	R_{\ell_{0-1}}(f_Q)=\E_{Y}\ell(Y,Y_Q)=1-\sum_{i\in C}\Pr\{Y=i\}\Pr_Q\{Y_Q=i\}.		
	$$
	Therefore,
	$$
	R_{\ell_{0-1}}(f_Q) \geq 1-\Pr\{Y=j\} \sum_{i\in C}\Pr_Q\{Y_Q=i\} = 1-\Pr\{Y=j\} = R_{\ell_{0-1}}(f_0) .
	$$

	%
	%
\end{proof}	

We conclude that the Prediction Advantage for multi-class classification under the 0-1 loss 
function (which includes binary classification) is
$$
\textnormal{PA}_\ell (f)\eqdef1-\frac{R_\ell(f)}{R_\ell(f_0)}=1-\frac{R_\ell(f)}{1-\max_{i\in C}(\Pr\{Y=i\})}.
$$

\subsection{Prediction Advantage for Squared Loss in Regression}
We now discuss a (multiple) regression setting under the squared loss.
For simplicity, we consider the univariate outcome model where $\cY = \reals$, but the results can be extended to multivariate outcome models $\cY = \reals^m$ in a straightforward manner.
Thus, the predicted variable is a real number,  $Y\in \reals$, and
the squared loss function is defined as $\ell(r,s)=(r-s)^2$.
We now show that the BMP strategy for this setting is $f_0\eqdef\E(Y)$.

\begin{lemma}[BMP for squared loss in regression]
	$f_0\eqdef\E[Y]$ is the BMP for  regression under the squared loss. 
\end{lemma}
The proof for this lemma is the known result of \emph{minimum mean squared error} in the field of signal processing, for further reading see  \cite{oppenheim2015signals} (chapter 8).

%
Having identified the BMP and observing that 
$$
R_\ell(f_0) = \E_Y [(Y - f_0)^2 ] = \E_Y [(Y - E[Y])^2 ] = \var(Y),
$$
we obtain the following expression for the PA:
$$
\textnormal{PA}_\ell (f)\eqdef1-\frac{R_\ell(f)}{R_\ell(f_0)}=1-\frac{R_\ell(f)}{\var(Y)}.
$$
Apparently, the PA in regression is precisely the well-known $R$-squared measure in regression analysis
(also known as the coefficient of determination) \cite{helland1987interpretation}.

\subsection{Prediction Advantage for Absolute Loss in Regression}
Consider a univariate (multiple) regression setting under the absolute loss.
The predicted variable is a real number,  $Y\in \reals$, and
the absolute loss function is defined as $\ell(r,s)=|r-s|$.
We now show that the BMP strategy for this setting is $f_0\eqdef\ median(Y)$.

\begin{lemma}[BMP for absolute loss in regression]
	$f_0 \eqdef median(Y)$ is the BMP for  regression under the squared loss. 
\end{lemma}
\begin{proof}
	According to Lemma~\ref{lemma:constant} the BMP in regression is a constant function, if we consider an arbitrary constant function $f_a=a \in \reals$, the risk of $f_a$ is 
	
	\begin{eqnarray*}
		R_{\ell}(f_a)&=&\E_Y|Y-a|\\
		&=&\int_{Y}|Y-a|P(Y)dP(Y)\\
		&=&\int_{-\infty}^{a}-(Y-a)P(Y)dP(Y)+\int_{a}^{\infty}(Y-a)P(Y)dP(Y).\\
	\end{eqnarray*}
	Taking the derivative with respect to $a$ to find the risk minimizer, we have
	$$
	\frac{\partial R_{\ell}(f_a)}{\partial a}=\int_{-\infty}^{a}P(Y)dP(Y)-\int_{a}^{\infty}P(Y)dP(Y)=0 .
	$$
	Thus,  the BMP is $f_0=median(Y)$.
\end{proof}
Having identified the BMP, we find that the risk is the \emph{mean absolute deviation} (MAD) around the median,
$$
R_\ell(f_0) = \E_Y [|Y - median(Y)|]=D_{med}.
$$
We obtain the following expression for the PA:
$$
\textnormal{PA}_\ell (f)\eqdef1-\frac{R_\ell(f)}{R_\ell(f_0)}=1-\frac{R_\ell(f)}{D_{med}}.
$$

\subsection{Prediction Advantage for Cost-Sensitive Loss}
Consider a multi-class classification problem where the classes are 
$C \eqdef \{ 1, 2, \ldots, k\}$),  and  the loss function is defined with a specific cost for each misclassification type (see Elkan, (2001), for further details). 
For $1 \leq i, j \leq k$, $b_{i,j}$ 
is the cost for predicting label $i$ while the true label is $j$.
Clearly, the BMP in this setting is 
$$
f_0\eqdef \argmin_{i\in C}\left({\sum_{j\in C}{b_{i,j}\cdot{\Pr(Y=j)}}}\right).
$$
We skip the proof of optimality of this proposed BMP,  which is similar to the proof of Lemma~\ref{0-1lemma}.
%
The risk of the BMP is
$$
R_{\ell_c}(f_0)= \min_{i\in C}\left({\sum_{j\in C}{b_{i,j}\cdot{\Pr(Y=j)}}}\right),
$$

and therefore, the prediction advantage for the cost-sensitive loss is
$$
\textnormal{PA}_{\ell_c (f)}\eqdef 1-\frac{R_{\ell_c(f)}}{R_{\ell_c(f_0)}}
=1-\frac{R_{\ell_c(f)}}{\min_{i\in C}\left({\sum_{j\in C}{b_{i,j}\cdot{\Pr(Y=j)}}}\right)}.
$$

\section{Related Measures}
\label{sec:related work and other measures}

The false positive and false negative rates (called also type 1
and type 2 errors, respectively) are typically used in statistics in
the context of hypothesis testing. One
can meaningfully compare two classifiers by considering (separately)
their false positive and false negative rates. Moreover, Brodersen et al. \cite{brodersenbalanced} proposed to compare the mean of the true positive and true negative rates and defined it as the \emph{balanced accuracy}. 

Emerging from information retrieval, \emph{precision}
and \emph{recall} are two popular measures that can be used to meaningfully measure performance 
in imbalanced problems \cite{raghavan1989PR}. 
In the context of a binary classification problem with a target minority class, 
precision is the fraction of instances classified to be in the target that truly are in the target.
Recall is the fraction of true target instances that are classified correctly.
As stated in  Manning et al. \cite{manning2008introduction}, the use of these measures prevents meaningless assessments in imbalance problems. Precision and recall are scaled similarly and thus can be combined to a single quantity. It is common in text categorization to take their harmonic mean as a single performance
measure to be optimized. This quantity is called the
\emph{$F$-measure} and it approximates the so-called \emph{break-even point}, defined to be the equilibrium 
of the trade-off between precision and recall,


Another performance measure, which emerged  from experimental psychology, is \emph{Cohen's kappa}  \cite{cohen1960coefficient}. Originally, Cohen's kappa  was introduced as  a statistical measure to quantify 
the agreement between two raters, each of whom   classifies 
$N$ items into $C$ mutually exclusive categories, while taking into account
the probability that the raters agree by chance alone. 
Formally, it was defined as $\kappa \eqdef 1-\frac{1-p_0}{1-p_e}$, 
where $p_0$ is the fraction of items the raters agreed upon, and $p_e$ is the 
probability of chance agreement on a new item given that each rater only knows her own class distribution. Thus, $\kappa$ is a normalized measure, where $\kappa= 1$ represents complete agreement, and
$\kappa = 0$, complete disagreement.
Cohen's kappa has been advocated by  Fatourechi et al. \cite{fatourechi2008comparison} to tackle imbalanced classification problems.
To apply it in classification, we take one of the raters to be nature, which assigns true labels, and the other
rater is the classifier. Then, $p_0$ is the 0/1 accuracy of the classifier (over a sample of $N$ instances) and
$p_e$ is the (label) marginal of the classifier multiplied by the (label) marginal of nature. 
While the use of $\kappa$ to quantify agreement between raters is well motivated, its use in classification, 
where the classifier and nature do not play symmetric roles, is problematic.
Given a labeled training sample, the learner has some knowledge of
nature's marginal, which can be utilized to obtain higher chances to hit the true label. In contrast,
in the case of two symmetric raters, no one knows anything about the other's marginal. 
The prediction advantage measure make up this deficiency, and unlike Cohen's kappa, it also
applies to any loss function.


%
\emph{$R$-squared} 
is probably the most popular measure of fit in statistical modeling and particularly in the context of 
regression analysis \cite{helland1987interpretation}.  There have been numerous attempts to extend the $R$-squared measure to logistic regression.
For example,  Efron \cite{efron1978regression} extended the $R$-squared measure by calculating the squared loss over the predicted probabilities divided by the variance measured in the probability space and McFadden \cite{mcfadden1974conditional} extended the $R$-squared measure by replacing the loss with the log-likelihood of the model, and the variance with the log-likelihood of the intercept.
In fact, the large number of proposed performance measures
in the context of logistic regression is somewhat confusing and indicates a lack of wide consensus on how to extend the $R$-squared measure to this setting. We emphasize that all these proposals only relate to (logistic) regression
and generally are not defined for other loss functions.

\begin{figure}[htb]
	\centering
	\fbox{\rule[0cm]{0cm}{0cm}
	\subfigure[]{\includegraphics[width=0.3\linewidth]{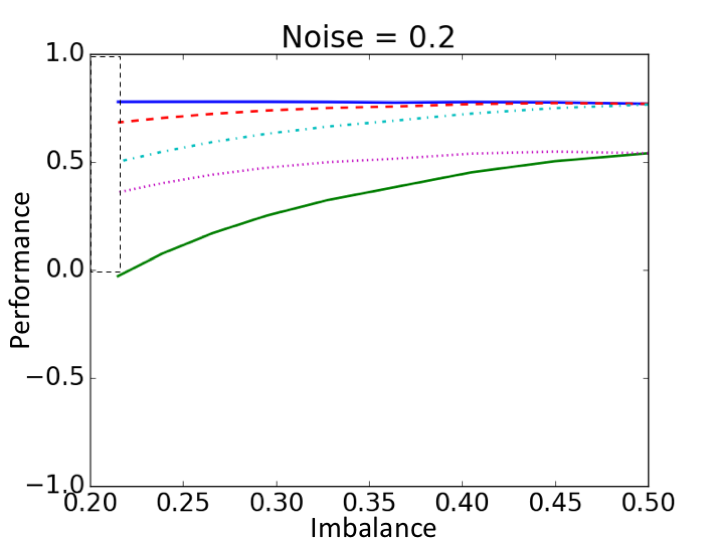} }
	\subfigure[]{\includegraphics[width=0.3\linewidth]{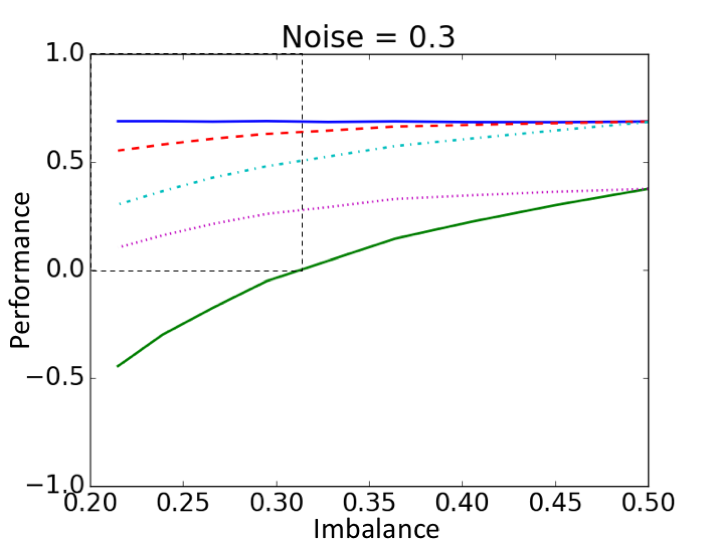} }
	\subfigure[]{\includegraphics[width=0.3\linewidth]{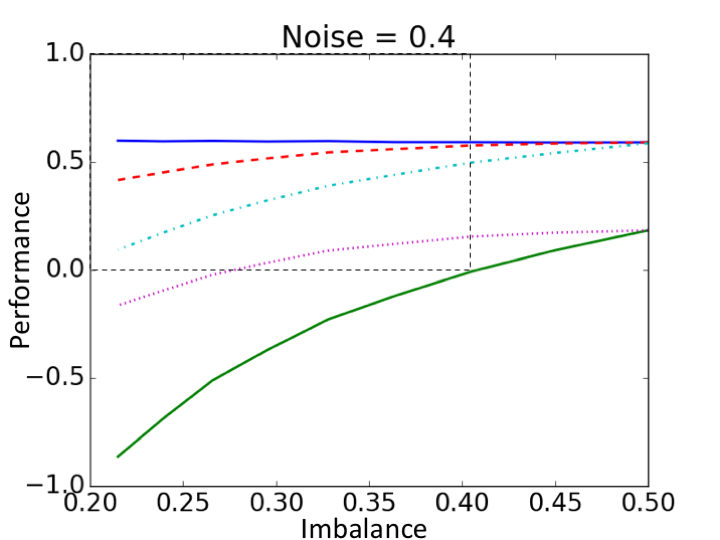} }\rule[0cm]{0cm}{0cm}}
	
	\caption{Comparison of performance measures for different levels of imbalance and noise on the breast cancer dataset. Blue -- accuracy, red -- balanced accuracy, turquoise -- F1 score, purple -- Cohen's Kappa, green -- Prediction Advantage.}
	\label{fig:binary-exp}
\end{figure}

\section{Analysis }
\label{sec:analysis}

When the PA of a prediction function is not positive, the BMP outperforms our function. In this case our function is possibly no better than trivial.
Thus, using the PA allows us to detect such trivial cases.
Most of the alternative methods mentioned in Section~\ref{sec:related work and other measures} are defined in terms
of the 0-1 loss function in a binary classification setting. 
For this setting we can show that the PA lower bounds all other measures.
Consequently, when the PA is zero, 
all the alternative measures will be positive and wrongly qualify meaningless functions.

The formal introduction of all the measures we discuss above can be made simple using 
the Venn diagram in Figure~\ref{fig:diagram}. Let $f$ be a classifier whose performance we would like 
to quantify. The areas in this diagram are defined as follows.

\begin{eqnarray*}
	a &\eqdef& \Pr_{X,Y} \{ f(X) = 1, Y = -1\},\\
	b &\eqdef& \Pr_{X,Y} \{ f(X) = 1, Y = 1\},\\
	c &\eqdef& \Pr_{X,Y} \{ f(X) = -1, Y = 1\},\\
	d &\eqdef& \Pr_{X,Y} \{ f(X) = -1, Y = -1\}.
\end{eqnarray*}

We can express the PA
through the 0-1 errors of the BMP ($f_0$) and our function ($f$)
as follows.

\begin{eqnarray*}
	\ell(f_0(X),Y) &=& b + c\\
	\ell(f(X),Y) &=& a+ c\\
	\textnormal{PA}_\ell(f(X),Y) &=& 1 - \frac{\ell(f(X),Y)}{\ell(f_0(X),Y)} = 1 - \frac{a+c}{b+c} = \frac{b-a}{b+c}.
\end{eqnarray*}

\begin{figure}[htb]
	\centering
	\includegraphics[width=0.7\linewidth]{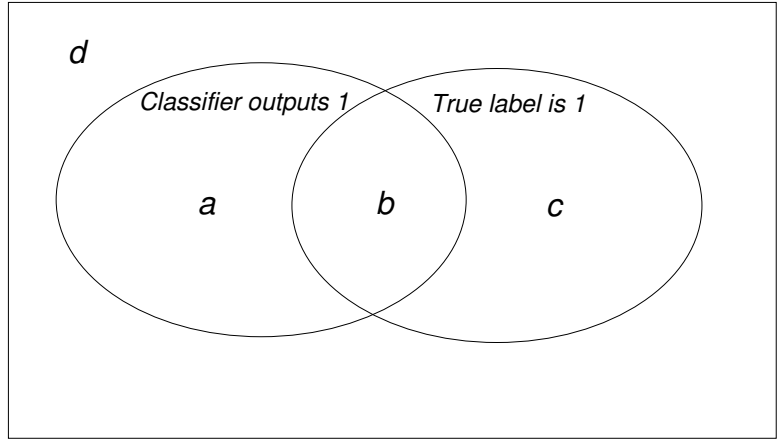}
	\caption{ Venn diagram: The areas $a,b,c$ and $d$ form a partition of the rectangle
		(i.e., they are mutually exclusive and exhaustive). We assume that $a+b+c+d = 1$.
		We also assume w.l.o.g. that the label '1' represents the minority class,
		and thus  $b+c \leq a + d$.}
	\label{fig:diagram}
\end{figure}

The true positive rate $TP(f(X),Y)$ and the true negative rate $TN(f(X),Y)$ are:
\begin{eqnarray*}
	TP(f(X),Y) &\eqdef& \Pr_{X,Y}\{ f(X)=1 | Y = 1\} = \frac{b}{b+c}\\
	TN(f(X),Y) &\eqdef& \Pr_{X,Y}\{ f(X)=-1 | Y = -1\} = \frac{d}{a+d}.
\end{eqnarray*}
Recalling that the balanced accuracy (see Section~\ref{sec:related work and other measures}) is defined 
as the arithmetic mean of $TP$ and $TN$, we can state the following bounds.

\begin{lemma}
	\label{lemma:lemma TP-TN}
	For any classifier $f$,
	\begin{eqnarray}
	\textnormal{PA}_\ell(f(X),Y) &\leq&  TP(f(X),Y) \label{eq:TP}\\
	\textnormal{PA}_\ell(f(X),Y)&\leq& TN(f(X),Y)\label{eq:TN}\\
	\textnormal{PA}_\ell(f(X),Y)&\leq& \frac{TN(f(X),Y)+TP(f(X),Y)}{2} \eqdef BA(f(X),Y)\label{eq:BA} ,
	\end{eqnarray} 
	with strict inequality in~(\ref{eq:BA}) when $ TP(f(X),Y)<1$ or $TN(f(X),Y)<1$ (i.e., the 0/1 error 
	of $f$ is not zero).

\end{lemma}

\begin{proof}
	Clearly,
	$$
	\textnormal{PA}_\ell(f(X),Y)=\frac{b-a}{b+c}\leq \frac{b}{b+c}=TP(f(X),Y),
	$$
	which proves~(\ref{eq:TP}), and clearly, 
	when $a>0$,~(\ref{eq:TP}) holds with strict inequality.
	
	To prove~(\ref{eq:TN}), we first note that since the label '1' represents the minority class,
	$b+c \leq a + d$. Therefore,
	$$
	\textnormal{PA}_\ell(f(X),Y)=\frac{b-a}{b+c}=1-\frac{a+c}{b+c}\leq 1-\frac{a}{b+c}\leq 1-\frac{a}{a+d}=\frac{d}{a+d}=TN(f(X),Y),
	$$
	and whenever $c>0$, (\ref{eq:TN}) holds with strict inequality. 
	
	Obviously,  (\ref{eq:TP}) and~(\ref{eq:TN}) together imply  ~(\ref{eq:BA}), and 
	whenever~(\ref{eq:TP}) or~(\ref{eq:TN}) hold with strict inequality, (\ref{eq:BA}) also holds with strict inequality.
\end{proof}

Turning now to precision and recall and the related F-measure (and break-even point),
the \emph{precision} $PRE(f(X),Y)$ and \emph{recall} $REC(f(X),Y)$ can be expressed as
$$
PRE(f(X),Y) \eqdef \frac{b}{a+b}    \ \ \ \ \
REC(f(X),Y) \eqdef \frac{b}{b+c}.
$$
The F-measure is
$$
F(f(X),Y) \eqdef \frac{2}{1/PRE(f(X),Y) + 1/REC(f(X),Y)} = \frac{2 PRE(f(X),Y) REC(f(X),Y)}{PRE(f(X),Y) + REC(f(X),Y)}
.$$
Clearly, a trivial classifier $f$ that classifies everything as `0' will have a zero
F-measure.

The following lemma shows the strict domination of the PA relative to precision and recall and the derived F-measure and BEP.

\begin{lemma}
	\label{lemma:lemma pre-rec}
	For any classifier $f$,
	\begin{eqnarray}
	\textnormal{PA}_\ell(f(X),Y) &\leq& PRE(f(X),Y)\label{eq:PRE}\\
	\textnormal{PA}_\ell(f(X),Y) &\leq& REC(f(X),Y)\label{eq:REC}\\
	\textnormal{PA}_\ell(f(X),Y) &\leq& F(f(X),Y),
	\label{eq:Fmeasure}
	\end{eqnarray}
	and when $PRE(f(X),Y)<1$, (\ref{eq:Fmeasure}) holds with strict inequality.
	
\end{lemma}

\begin{proof}
	Note that
	it is always true that $-a^2 \leq 0 \leq bc$. Adding $b^2$ to both sides, we get
	$$
	(b+a)(b-a) = b^2-a^2 \leq bc+b^2 = b(b+c).
	$$
	Dividing both sides by $(b+c)(a+b)$, we get
	$$
	\textnormal{PA}_\ell(f(X),Y) = \frac{b-a}{b+c} \leq \frac{b}{a+b} = PRE(f(X),Y).
	$$
	
	Proving that $\textnormal{PA}_\ell(f(X),Y) \leq REC(f(X),Y)$ is immediate:
	$$
	\textnormal{PA}_\ell(f(X),Y) = \frac{b-a}{b+c} \leq \frac{b}{b+c} = REC(f(X),Y).
	$$
	Whenever $a>0$,~(\ref{eq:PRE}) and~(\ref{eq:REC}) hold with strict inequality.
	
	Obviously,  if both (\ref{eq:PRE}) and~(\ref{eq:REC}) hold, (\ref{eq:Fmeasure}) also holds, and whenever $PRE(f(X),Y)<1$,  both (\ref{eq:PRE}) and (\ref{eq:REC}) hold with strict inequality, in which case (\ref{eq:Fmeasure}) also holds with strict inequality.
\end{proof}

Turning now to the Cohen's kappa (see Section~\ref{sec:related work and other measures}), we 
first express it in terms of the areas. Let $p_0=b+d$ and $p_e=(a+b)(b+c)+(a+d)(c+d)$. We get that the kappa is,

$$
\kappa(f(X),Y) \eqdef 1-\frac{1-p_0}{1-p_e}=1-\frac{a+c}{1-(a+b)(b+c)-(a+d)(c+d)}.
$$
The following lemma shows that the PA lower bounds Cohen's kappa.

\begin{lemma}
	\label{lemma:lemma kappa}
	For any classifier $f$,
	\begin{eqnarray*}
		\textnormal{PA}_\ell(f(X),Y) &\leq& \kappa(f(X),Y) \label{eq:kappa},
	\end{eqnarray*}
	with strict inequality whenever the problem is imbalanced.
	
\end{lemma}

\begin{proof}
	Recalling that $b+c$ is the mass of the minority class, and thus $b+c \leq a+d$,
	$$
	p_e  = (a+b)(b+c)+(a+d)(c+d) \leq (a+b)(a+d)+(a+d)(c+d)=a+d,
	$$
	and when the problem is imbalanced, namely $b+c < a+d$,
	$p_e < a+d$.	Therefore,
	\begin{eqnarray*}
		\kappa(f(X),Y) &=& 1-\frac{a+c}{1-(a+b)(b+c)-(a+d)(c+d)}\\
		&\geq& 1-\frac{a+c}{1- (a+d)}\\
		&=&1-\frac{a+c}{b+c}\\
		&=&\textnormal{PA}_\ell(f(X),Y),
	\end{eqnarray*}
	and clearly, strict inequality will be obtained if the problem is imbalanced.
\end{proof}

%

\section{Numerical Examples}
\label{sec:experiments}

The main benefit in using the PA is its ability to detect meaningless performance of 
prediction functions relative to the Bayesian marginal prediction. 
In this section we first show a number of numerical examples that highlight the advantage of using  the PA in certain cases of imbalance and noise. 

Consider the following situation: You have purchased a classifier $f$ whose 0-1 accuracy is 
claimed by the factory to be 70\%. Is this classifier meaningful for your problem?
Unfortunately, it may very well be the case that you would be better off using the 
Bayesian marginal classifier (which you can easily train on your own based on
label proportions). However, if you are using the 0-1 loss, you would have no way of knowing this. Using the PA, however, would allow you to detect this situation. Some alternative measures, such as 
Cohen's kappa, can also help somewhat. 
In this section we empirically rank the various performance measures discussed above and 
show several values of imbalance and noise where only the PA will detect meaningless classifiers.

To this end, we consider the UCI breast cancer dataset \cite{Lichman:2013}, which is nearly balanced and 
nearly realizable.\footnote{For example, we trained
	a random forest classifier for this set whose test 0-1 error is $\hat{R}_\ell=0.038$.}
We use this set to synthetically generate a sequence of imbalanced and noisy independent 
test cases on which 
we compare all performance measures mentioned in Section~\ref{sec:related work and other measures}. We control imbalance by synthetically inflating class proportions to desired 
levels using bootstrap
sampling.
Label noise is also controlled to desired levels by flipping the labels of randomly selected 
equal size subsets of both the minority and majority classes.

Consider Figure \ref{fig:binary-exp}, where we show performance levels of the various measures
on a performance-imbalance plane for several noise levels.
For example, in \ref{fig:binary-exp}(a) we consider 20\% label noise. It is evident that the PA (green curve) lower bounds all
other measures. Moreover, all performance measures are always montonically ordered 
as follows: the PA (green) is the lowest, and above it are Cohen's kappa (purple) , the F-measure (turqoise) , the balanced accuracy (red), and accuracy (blue). It is possible to prove that
this order always holds.
The interesting region in this graph is the top left quadrant defined by the intersection of the 
PA with zero. In this region, all the other measures indicate that the classifier at hand is 
useful even though it in fact achieves lower performance than the BMP. 
While Cohen's kappa can detect some of this triviality, it also falls into this trap in many 
cases. 
We can also see that when the problem is balanced, the PA and Cohen's kappa are equal and, when the problem has any level of noise and imbalance, the PA is strictly lower than all the other measures and will be the first to detect trivial solutions.
Taken together with the proofs given in Section~\ref{sec:analysis} (showing relationships of the PA to 
all the other measures),
the above support the claim that the PA is a favorable measure for imbalance problems.

Recall the numerical example from Section~\ref{sec:introduction}. We use the prediction advantage to measure the performance of Bob and Alice in the exam. Bob received a grade of 60 on exam A where there were three optional answers for each question. His loss is $\ell(Bob(A))=0.4$. The loss of the BMP for this test is $\ell(f_0(A))=\frac{2}{3}$ , leading to a PA for Bob of $PA_{Bob}=1-0.4/\frac{2}{3}=0.4$. Calculating the same for Alice on the second test yields $\ell(Alice(B))=0.4$, $\ell(f_0(B))=0.75$, and the PA is $PA_{Alice}=1-0.4/0.75=\frac{7}{15}=0.46$. We see that when using the PA, the grades that were indistinguishable  are now distinguishable: the PA takes into account the complexity of the problem represented by the marginal distribution of $Y$, and by doing so gives more information. 
\section{Concluding Remarks}
We proposed a general performance measure to quantify prediction quality.
The measure  quantifies prediction quality regardless of 
problem-dependent distortions such as class imbalance, noise, variance and number of classes.
Unlike previous methods, the proposed measure is defined for any loss function,
and we derived simple formulas for all popular loss functions.
In the case of the squared loss function, the well-known $R$-squared measure emerged.

One reason for the popularity of the $R$-squared measure in regression analysis is its ability to 
measure performance in a ``unified'' manner across problems. This attractive property is obtained 
by normalizing the risk of the regression function by the variance of the marginal outcome distribution.
PA generalizes this attractive property of the $R$-squared measure to any supervised learning problem under any
loss function and enjoys the same uniformity property across problems.
In general, the normalizing factor of any problem is captured by the performance of the Baysian
marginal prediction (BMP). For non-trivial solutions, the PA uniformly quantifies 
the quality in percentage units on a scale where 0 represents triviality (BMP) and 1 represents
perfect prediction. PA has no disadvantages whatsoever even if the problem is binary and perfectly balanced.
To prevent trivial solutions and allow for comparisons across problems, we therefore recommend using the PA in all cases.

While in this paper we advocate the use of the PA mainly as a means to prevent triviality in imbalanced problems, a very important application of the PA would be \emph{selective prediction}
(otherwise known as classification with a reject option) 
\cite{chow1970optimum,el2010foundations,wiener2015agnostic,2017GelbhartE}, where the goal 
is to abstain over a subset of the domain so as to minimize risk in the covered region.
In this setting, the PA is essential.
For example, consider an imbalanced problem with a minority class consisting of 
20\% of the domain. Two selective classifiers were trained to cover 70\% of the domain
and both achieved 5\% risk at this coverage.
It is not hard to see that it is meaningless to compare these classifiers based on risk alone because their class distribution is different; e.g., only one of these classifiers rejected the minority class 
in its entirety, thus giving a trivial solution. 
The interesting open question in this regard would be how to use the PA as an optimization criterion for learning optimal selective prediction functions.

\section*{Acknowledgments}
This research was supported by The Israel Science Foundation (grant No. 1890/14)

\bibliographystyle{plain}  
\bibliography{bib1}

\appendix

\section{Prediction Advantage on Haberman Dataset}
\label{app:haberman}
The widely used Haberman dataset contains cases from a study that was conducted between 1958 and 1970 at the University of Chicago's Billings Hospital on the survival of patients who had undergone surgery for breast cancer \cite{Lichman:2013}. 
The dataset represents a small binary classification problem where
the proportion of the minority class is  26.47\% . 
Quite a few papers 
reported on classifiers they have trained for this dataset whose 0/1 test errors 
are close or no better than the trivial classifier that always classifies according to the majority label. To convincingly demonstrate that this is indeed a widely occurring problem, in this 
appendix we briefly provide the details of a few such results.

\cite{wiener2011agnostic} used SVM with a reject option and reported on 0.27 0/1 error without any rejection; this result is equivalent to a PA of -0.02. 
\cite{huang2007correcting} trained a classifier whose error of 0.3 with a corresponding PA of -0.13. 
\cite{xiong2012random} reported on a classifier whose error is 0.273 (PA $=  -0.0313$), and \cite{chandra2009fuzzifying} reported on 0.2742 error whose PA is negative as well. \cite{ghanavati2014effective} compared several methods to learn imbalanced problems. A few of the proposed methods suffered 0/1 errors greater than 0.3; the use of the $F$-measure in this study 
couldn't help detecting these problematic results.
\cite{mcconnell2004building} reached 0/1 accuracy of 73.4 (0/1 error 0.266) corresponding again to a negative PA. \cite{juhola2013missing} reached accuracy below 70\% in almost all of their experiments, and the only method achieving a non-trivial solution was a naive Bayes classifier (in all the other cases the PA was negative). 
\cite{chandra2006robust} reported on a maximal 0/1 test accuracy of 71.7\% , and \cite{alabdulmohsin2014support} reported on 26.6\% error, both corresponding to
a negative prediction advantage.

\end{document}